\documentclass{article} 
\usepackage[final]{colm2026_conference}

\usepackage{microtype}
\usepackage{hyperref}
\usepackage{url}
\usepackage{booktabs}
\usepackage{bbm}
\usepackage{amsmath}
\usepackage{amsthm}
\newtheorem{proposition}{Proposition}
\newtheorem{remark}{Remark}
\usepackage{multirow}
\usepackage{graphicx}
\usepackage{subcaption}
\usepackage[table]{xcolor}

\usepackage{lineno}

\definecolor{darkblue}{rgb}{0, 0, 0.5}
\hypersetup{colorlinks=true, citecolor=darkblue, linkcolor=darkblue, urlcolor=darkblue}

\title{Weightless Fine-Tuning: Personalizing LLMs via Logit-Space Transport}

\usepackage{graphicx}

\author{
\resizebox{0.95\textwidth}{!}{
\begin{tabular}{cccc}
\bfseries Bohan Zhang
&
\bfseries Anqi Ni
&
\bfseries Yixin Wang
&
\bfseries Paramveer S. Dhillon
\\
\normalfont University of Michigan
&
\normalfont University of Chicago
&
\normalfont University of Michigan
&
\normalfont University of Michigan
\\
\normalfont\texttt{zbohan@umich.edu}
&
\normalfont\texttt{anqini4@gmail.com}
&
\normalfont\texttt{yixinw@umich.edu}
&
\normalfont\texttt{dhillonp@umich.edu}
\end{tabular}
}
}

\begin{document}

\ifcolmsubmission
\linenumbers
\fi

\maketitle

\begin{abstract}
Supervised fine-tuning (SFT) is a standard approach for adapting LLMs to a target distribution, but in settings such as personalization, where each author requires separate weight access, optimization, storage, and retraining, its costs become prohibitive. We propose Weightless Fine-Tuning (WFT), a training-free decoding-time method that approximates the distributional effect of SFT without weight updates. WFT computes supervised residuals on an author's training sequence and transports them to the current prompt through a cross-prefix transport operator estimated from dropout-induced cross-covariance. The operator captures how a perturbation at one context propagates to predictions at another, replacing gradient-based parameter updates with logit-space corrections. On three LaMP personalization benchmarks, WFT achieves the best average performance across datasets, matches or exceeds SFT on individual tasks, and outperforms other lightweight baselines on average. In a budget-controlled comparison, WFT approaches SFT performance using less than 7\% of the effective computation. Logit-level analysis shows a cosine similarity of 0.875 between the logit shifts induced by WFT and SFT over 95\% of the next-token probability mass, suggesting that WFT captures the distributional effect of supervised adaptation without modifying model weights.
\end{abstract}

\section{Introduction}
Large language models (LLMs) have achieved impressive performance across a wide range of tasks, yet many real-world applications require outputs that are tailored to individual authors or domains rather than a single generic response style. Personalized generation has therefore become an important setting for LLM deployment. In such settings, the strongest adaptation quality often comes from supervised fine-tuning (SFT) on author-specific data~\citep{chakrabarty2025readerspreferoutputsai}. The practical value of SFT lies in the output distributions induced by the updated weights. When a language model is fine-tuned on an author’s writing, the objective is to shift its next-token predictions toward patterns characteristic of that author rather than to preserve a particular parameter configuration. Yet the dominant method for producing this shift still relies on gradient-based parameter updates, which require costly training. In personalization, the cost scales poorly. If hundreds or thousands of authors each need their own adaptation, practitioners must run separate fine-tuning jobs, store separate weights, and retrain the model whenever new author data arrives.
 
Existing approaches span a wide spectrum, but each comes with clear limitations. Full SFT remains a strong quality benchmark for personalized generation~\citep{salemi-etal-2024-lamp,chakrabarty2025readerspreferoutputsai,chakrabarty2026goodwritinggenerativeexpertlevel}, but its computational and storage costs scale linearly with the number of authors, and training on narrow personal data risks catastrophic forgetting of general capabilities~\citep{kaushik2021understandingcatastrophicforgettingremembering}. Parameter-efficient personalization methods reduce the number of author-specific parameters while still requiring backpropagation~\citep{liu-etal-2023-recap,huber2025embeddingtoprefixparameterefficientpersonalizationpretrained}. Personalization here is typically achieved through learned prefix representations trained across many authors. Their performance depends on how well those representations generalize across authors, which becomes less reliable when author populations are highly diverse. In-context prompting avoids training altogether by prepending a small number of author examples to the prompt. Yet it conditions on those examples without explicitly computing the residual between the model’s current prediction and the desired target, and its performance typically remains below methods that perform supervised adaptation~\citep{salemi-etal-2024-lamp,zhang2025personalize}.
 
Recent work on decoding-time adaptation suggests that the community is increasingly aware of this tension. Drift~\citep{kim2025drift} aligns decoding with user preferences through interpretable attributes. Amulet~\citep{zhang2025amulet} treats token prediction as an online learning problem, and CoSteer~\citep{lv2025costeer} uses logit deltas from a local model to steer a larger one. In parallel, test-time training approaches~\citep{sun2024learning, hu2025test} push adaptation to the test phase but still rely on lightweight parameter updates such as LoRA fine-tuning. In-context vectors~\citep{icvliu2024} extract task-relevant directions from hidden activations and use them to steer model behavior at inference time without updating model parameters. CHAMELEON~\citep{zhang2025personalize} generates synthetic author preference data from limited author history and performs inference-time representation editing for scalable personalization. These efforts collectively demonstrate that adaptation need not happen before deployment, but none of them directly addresses a more fundamental question: \emph{Can we approximate the change in next-token distributions that SFT would produce, without updating any parameters at all?}

In this paper, we answer this question with \textbf{Weightless Fine-Tuning (WFT)}, a training-free fine-tuning method that operates entirely in logit space. The key idea is to compute standard supervised residuals on a given training sequence, and then \emph{transport} their effect to the current prompt prefix through a cross-prefix transport operator $M$. This operator, estimated via the cross-covariance of dropout-perturbed forward passes, serves as an empirical approximation to the cross-prefix neural tangent kernel (NTK) and captures how a perturbation at one training context propagates to predictions at another through weights. The result is a lightweight procedure that approximates the distributional effect of SFT without modifying the model's weights.
 
We evaluate WFT on three generative personalization tasks from the LaMP benchmark, including paper title generation, news headline generation, and tweet paraphrasing. WFT achieves the best average performance across datasets on both foundation models, while outperforming or remaining highly competitive with SFT on individual tasks. It also consistently outperforms other lightweight personalization baselines on average. Under budget-controlled comparison, WFT reaches performance comparable to SFT while using less than 7\% of the effective computation. Moreover, ablation and qualitative logit-level analyses support the view that WFT serves as a practical training-free approximation to supervised adaptation.

Our contributions are as follows:
\begin{itemize}
    \item We propose Weightless Fine-Tuning, a training-free method that approximates SFT in distribution space by transporting supervised residuals across prefixes via a dropout-estimated cross-prefix transport operator.
    \item We conduct comprehensive experiments on three personalization benchmarks, showing that WFT achieves the best average performance across datasets, remains competitive with SFT on each task, and outperforms other lightweight personalization baselines.
\end{itemize}

\section{Related Work}

Personalization of LLMs has attracted growing attention. Recent surveys have summarized the literature from several perspectives.~\citet{chen2024large} discussed how LLM-based personalization extends earlier paradigms of passive information filtering toward more active forms of author engagement.~\citet{zhang2024personalization} organized existing methods into retrieval-augmented generation, prompt engineering, representation learning, and RLHF.~\citet{liu2025survey} categorized personalization methods by the stage at which they are introduced, including input-level prompting, model-level fine-tuning, and objective-level alignment. WFT spans these categories. It avoids parameter updates, as prompting-based methods do, while using supervised residuals to correct the model’s output distribution in a way that resembles fine-tuning.
 
Per-author SFT remains a strong personalization baseline.~\citet{chakrabarty2025readerspreferoutputsai} showed that fine-tuning on an author’s full corpus can generate text that readers judge to match that author’s style closely. The cost of maintaining a separate model for each author, however, has motivated more efficient alternatives. One line of research builds on prefix-tuning~\citep{li-liang-2021-prefix}, which introduced continuous prefix tokens for steering model behavior while keeping the backbone frozen. In personalized generation, RECAP~\citep{liu-etal-2023-recap} combines retrieval with a context-aware prefix encoder for dialogue, and Embedding-to-Prefix~\citep{huber2025embeddingtoprefixparameterefficientpersonalizationpretrained} maps pre-learned author embeddings to soft prefix tokens. These methods reduce per-author training costs, but they still rely on gradient-based optimization over large multi-author corpora and perform well only when the training population is generalizable to test authors.
 
Recent work has increasingly moved adaptation to inference time. Some methods perform lightweight parameter updates at test time. Test-Time Training~\citep{sun2024learning} updates the model on each test instance, and later work extends this idea to language models with LoRA-based updates~\citep{hu2025test}. Other methods modify the output distribution during decoding without changing model parameters. Drift~\citep{kim2025drift} aligns outputs with author preferences through interpretable attributes. Amulet~\citep{zhang2025amulet} treats each decoding step as an online learning problem for real-time preference adaptation. CoSteer~\citep{lv2025costeer} steers a cloud-hosted model with logit deltas from a locally adapted smaller model, and CHAMELEON~\citep{zhang2025personalize} combines synthetic preference data with inference-time representation editing. WFT derives its logit corrections from supervised residuals and natural gradient descent. In doing so, it aims to reproduce the distributional effect of SFT rather than to steer generation toward a particular attribute.
 
Theoretical work has examined links between in-context learning and gradient-based optimization.~\citet{von2023transformers} showed that linear self-attention can implement gradient descent in the forward pass, and~\citet{dai2023can} analyzed ICL in a dual form that connects it to implicit fine-tuning. Later studies placed clearer limits on this equivalence.~\citet{shen2023pretrained} found substantial gaps between ICL and gradient descent in pretrained language models on realistic NLP tasks, and~\citet{deutch2024context} questioned whether the connection holds beyond carefully controlled settings. WFT builds on the mathematical connection between contextual conditioning and optimization to define a training-free procedure that approximates the distributional effect of SFT. This formulation does not assume that in-context learning itself implicitly performs gradient descent through the forward pass.

\section{Weightless Fine-Tuning}

We present Weightless Fine-Tuning as a training-free approximation to SFT in distribution space.  The starting observation is that SFT changes the next-token distribution only \emph{indirectly}: a supervised gradient step modifies the model's weights, which in turn shifts the output logits, which finally alters the predictive distribution. Because this causal chain terminates in logit space, we can short-circuit it by computing the logit shift that would result from a weight update and applying it directly, without ever modifying the weights themselves. Concretely, WFT shows that the effect of one or multiple supervised natural-gradient updates on a training sequence can be transferred to a different prompt context through a cross-prefix transport operator $M$, yielding a next-token distribution that approximates the one fine-tuning would have produced. Hence, WFT directly modifies the logits at generation time without updating model parameters.

Let $z(c)\in\mathbb{R}^V$ denote the next-token logits of a frozen language model for a prefix $c$, and $p(c)=\mathrm{softmax}(z(c))$ denote the corresponding probabilities. 

We distinguish two sequences: 1) \textbf{Prompt prefix}: $x_{1:t}$. This is usually the task prompt, for example: ``Generate a title for the abstract of a research paper:'' 2) \textbf{Training sequence}: $\tilde{x}_{1:S}$. This could be potential training data used in SFT training, such as the author’s previous research paper titles. For the training sequence, the next-token target at position $s$ is $\tilde{x}_{s+1}$. We define the standard supervised residual as
\[
r_s = e_{\tilde{x}_{s+1}} - p(\cdot \mid \tilde{x}_{1:s}),
\]
where $e_{\tilde{x}_{s+1}} \in \mathbb{R}^V$ is the one-hot vector of the ground-truth next token. Intuitively, $r_s$ measures how the model’s current prediction differs from the supervised target at training position $s$. 

A KL-proximal natural-gradient (NGD) update with step size $\eta$ in logit direction $g$ solves
\[
q = \arg\max_r \left\{\eta\,\mathbb{E}_r[g] - \mathrm{KL}(r\|p)\right\}
\quad\Longrightarrow\quad
q \propto p\,e^{\eta g}.
\]
NGD is thus multiplicative in probability and additive in logits: if a gradient step on weights induces a logit shift $\Delta z_t$ at prompt time $t$, adding that shift to the original logits exactly recovers the post-weight-update distribution. This additivity is what makes a weight-free emulation possible.

\subsection{Cross‑Prefix Transport Operator and Its Estimation by Dropout Covariance}\label{sec:M}
An NGD step with residual $r_s$ induces a weight parameter change $\Delta\theta$ that, under local linearization, perturbs the logits at any prefix $c$ by $J_c\Delta\theta$, where $J_c$ is the Jacobian of the logits with respect to the weights. The weights thus serve as a \emph{medium} through which a training-side correction propagates to the prompt side. WFT replaces this implicit, weight-mediated propagation with an explicit cross-prefix transport operator $M_{t\leftarrow s}\in\mathbb{R}^{V\times V}$ that directly maps the residual at training position $s$ to the logit change at prompt position $t$. Concretely, we write the one-step logit update at prompt time $t$ as
\[z_t^{(1)} = z_t^{(0)} + \eta \sum_{s=1}^{S} M_{t \leftarrow s} r_s.\]
Equivalently, the induced logit correction is
$\Delta z_t = \eta \sum_{s=1}^{S} M_{t \leftarrow s} r_s.$

Thus, $M_{t \leftarrow s}$ is the core object in WFT: it tells us how an SFT-style correction computed on the training sequence should be transported to the current decoding prefix, capturing the role the shared weights would have played, but without touching them.

To estimate $M_{t \leftarrow s}$ without access to gradients of the full fine-tuning process, we rely on stochastic forward passes induced by dropout. We enable dropout at inference time and run $K$ forward passes on both the prompt prefix $x_{1:t}$ and the training prefix $\tilde{x}_{1:s}$. For each index \(k\), the prompt and training prefixes use the same dropout realization, while different indices use sampled masks independently. Let
\[
t_k = z^{[k]}(x_{1:t}) - \overline{z(x_{1:t})},
\qquad
s_k = z^{[k]}(\tilde{x}_{1:s}) - \overline{z(\tilde{x}_{1:s})},
\]
where $z^{[k]}(\cdot)$ denotes the logits from the $k$-th dropout realization (the dropout parameters are fixed in each pass but different across passes), and the bar denotes the empirical mean over the $K$ samples. We then form the empirical cross-covariance and self-covariance:
\[
\Sigma_{t,s} = \frac{1}{K-1} \sum_{k=1}^{K} t_k s_k^\top,
\qquad
\Sigma_{s,s} = \frac{1}{K-1} \sum_{k=1}^{K} s_k s_k^\top.
\]
Using a ridge-regularized linear estimator, we define
\[
\widehat{M}_{t \leftarrow s}
=
\Sigma_{t,s}(\Sigma_{s,s} + \lambda I)^{-1}.
\]

This estimator can be interpreted as an empirical approximation to the cross-prefix NTK action: it measures how a small perturbation at the training prefix is propagated to the prompt prefix in the local linear regime. Proposition~\ref{prop:ntk} formalizes this, showing that $\widehat{M}_{t\leftarrow s}$ converges to a ridge-regularized form of the true cross-prefix NTK action.

\subsection{Updating the Decoding Logits Using $M$}
Once $\widehat{M}_{t \leftarrow s}$ is available, we can compute the prompt-side logit shift directly from the supervised residuals as defined in Sec.~\ref{sec:M}. Under the assumptions of Proposition~\ref{prop:ntk}, the updated distribution approximates the one obtained from a one-step supervised natural-gradient update, so WFT can be viewed as a training-free emulator of one-step SFT in distribution space. Under multiple steps of update, to ensure consistency with the observed
update at arbitrary step $j$, we require:
\[
z^{(j+1)}_t = z^{(j)}_t + \eta\sum_s M^{(j)}_{t\leftarrow s} r^{(j)}_s,
\]
where the supervised residual $r_s^{(j)} = y_s - \mathrm{softmax}(z_s^{(j)})$
depends on the current training-side logits $z_s^{(j)}$, which are updated by
the local residual gradient:
$
z^{(j+1)}_s = z^{(j)}_s + \eta\, r^{(j)}_s.
$
The residual is therefore recomputed exactly at each step from the updated
logits. Therefore, after $k$ steps, the total prompt-side logit shift is
\[
\Delta z_t^{(k)}
=
\eta \sum_{j=0}^{k-1} \sum_{s=1}^{S}
M_{t \leftarrow s}^{(j)}\, r_s^{(j)}.
\]
Adopting a time-homogeneous approximation $M_{t \leftarrow s}^{(j)} \equiv M_{t \leftarrow s}$, this simplifies to
\[
\Delta z_t^{(k)}
=
\eta \sum_{s=1}^{S} M_{t \leftarrow s}
\sum_{j=0}^{k-1} r_s^{(j)}.
\]

This approximation is not exact because the operators can change as the logits evolve. When the updates remain small, however, fixing them still gives a useful first-order approximation and leads to a closed-form expression for the accumulated logit shift. Also, we write \(M_{t\leftarrow s}\) as a \(V\times V\) operator for clarity, but WFT does not materialize this large matrix in the implementation. The practical algorithm uses the vocabulary-coordinate-wise
dropout-factored approximation described in Appendix~\ref{app:acc_update}. Overall, WFT replaces the weight-update pathway of SFT with logit-space transport of supervised residuals. Its key object is the operator $M$, which specifies how learning signals from a training sequence modify predictions on a different prompt. This provides a lightweight approximation to SFT without changing a single model weight.

\section{Experiments}

\subsection{Datasets}
We evaluate our method using the LaMP benchmark~\citep{salemi-etal-2024-lamp}, which collects different personalization tasks. Our evaluation focuses on three generative datasets in this benchmark: 1) Personalized News Headline Generation. 2) Personalized Paper Title Generation. 3) Personalized Tweet Paraphrasing. Each author in the dataset has a certain amount of personal writing. For example, in the paper title generation task, an author has multiple historical pairs of paper abstracts and their corresponding titles. We follow the evaluation metrics used in the LaMP benchmark, namely ROUGE-1 (R-1) and ROUGE-L (R-L), between the generated outputs and the ground truth. More details about the dataset can be found in the Appendix~\ref{app:data}.
\subsection{Baselines}
We compare WFT with four baselines: 1) \textbf{SFT} is the standard approach for adapting an LLM to a target task or author distribution. 
Past research~\citep{chakrabarty2025readerspreferoutputsai} shows that per-author SFT on an author’s full corpus can generate text that human readers often judge as matching that author’s style well. This shows that per-author SFT is a strong baseline for personalization tasks. For each author, the model is tuned on author-specific examples using the next-token prediction objective. 
Given a personalized prompt–response pair, the model minimizes cross-entropy loss on the response tokens conditioned on the input context. 

2) \textbf{CHAMELEON}~\citep{zhang2025personalize} personalizes LLMs by first generating synthetic preference data from limited author history and then performing inference-time representation editing. It is a strong training-free personalization baseline on the LaMP dataset that avoids per-author fine-tuning while using only author-specific historical data. 

3) \textbf{Prefix-tuning-related} methods have shown strong performance across multiple personalization tasks. The core idea is to find a discriminative embedding for each author, for example, by using author profiles or by learning an author embedding in advance. Building on the representation learning strategy in~\cite{liu-etal-2023-recap} and the framework of~\cite{huber2025embeddingtoprefixparameterefficientpersonalizationpretrained}, we employ contrastive learning to derive author embeddings from historical texts and transform these embeddings into prefix tokens for personalized prefix-tuning to better adapt to the datasets.

4) \textbf{In-Context Prompting} has shown good baseline performance~\citep{salemi-etal-2024-lamp} in many personalization tasks. Following~\cite{chakrabarty2025readerspreferoutputsai}, the model is given five excerpts from the target author along with the specific task prompt. Based on these excerpts, the model generates a response that matches the author’s style. We use Qwen3-8B~\citep{yang2025qwen3technicalreport} and Llama-3.1-8B-Instruct~\citep{grattafiori2024llama} as the foundation models for all methods mentioned above. The details of the above methods are provided in Appendix~\ref{app:methods}.

\subsection{Budget Control Training}
SFT performance is heavily influenced by the training budget. As discussed in Section~\ref{sec:comp-analysis}, the budget required for SFT is significantly larger than that of WFT. Therefore, besides the performance comparison of the models under a sufficient budget as described above, we aim to examine how different budget levels affect SFT performance and whether SFT can maintain its performance when its budget is lowered to that of WFT. 

Intuitively, as the computational budget increases, the performance of SFT tends to improve. However, under a fixed training budget, prior work reports mixed results on whether it is better to use more training data or to spend the budget on more training steps~\citep{muennighoff2023scaling,kopiczko2026datarepetitionbeatsdata}. These outcomes likely depend on factors such as task type and the quality of the training data~\citep{yin2024compute-constrained,lagasse2025scalinglawtokenefficiency}. 

Therefore, we explore the relationship between budget and performance under two scenarios in the paper dataset: 1) When the ratio of the number of data points $N$ to the training epochs $E$ is small, i.e., a small dataset trained for many epochs, ensuring that the model fits the training data well. 2) When $N/E$ is large, i.e., a large dataset trained for relatively few epochs. In both scenarios, we vary the training budget $B$ by jointly adjusting $N$ and $E$ such that $B = N \times E$, and train per-author SFT models for each configuration. This allows us to trace performance as a function of B while comparing how different allocations affect the results. $N/E$ is set to 0.2 in the first scenario and 2 in the second scenario. The experiments in this section are conducted using Qwen3-8B.

\section{Empirical Studies}
\subsection{Main Results}
The performance of different methods using Qwen3-8B and Llama-3.1-8B-Instruct is shown in Tables~\ref{tab:results} and~\ref{tab:llama_results}, respectively.

On Qwen3-8B, WFT and SFT achieve the strongest results across the three LaMP datasets. WFT attains the best R-1 score on Twitter and the strongest performance on News, whereas SFT performs best on Paper and obtains the highest R-L score on Twitter. Averaged across datasets, WFT obtains the highest scores on both R-1 and R-L. WFT and SFT also outperform in-context prompting, prefix-tuning, and CHAMELEON on average. Although CHAMELEON improves over in-context prompting and prefix-tuning, it still falls short of WFT and SFT overall. On Twitter and News, prefix-tuning performs no better than in-context learning, likely because it struggles to learn effective author representations, as reflected by an authorship detection accuracy of only around 30\%.

On Llama-3.1-8B-Instruct, WFT again achieves the highest average R-1 and R-L scores. It obtains the best R-1 result on Paper and Twitter and the best R-L result on Paper, while remaining competitive on the other metrics. Prefix-tuning performs best on Twitter R-L and News R-1, and CHAMELEON achieves the highest
News R-L score. Overall, WFT retains the strongest average performance across the three datasets. The results across both foundation models show that WFT provides a strong training-free alternative to SFT.

\begin{table}[t]
\centering
\begin{tabular}{llccccc}
\toprule
Dataset & Metric & WFT & SFT & In-Context & Prefix-tuning & CHAMELEON \\
\midrule
\multirow{2}{*}{Paper} 
 & R-1 & \underline{0.440} & \textbf{0.461} & 0.372 & 0.427 & 0.404\\
 & R-L & \underline{0.398} & \textbf{0.399} & 0.311 & 0.372 & 0.339\\
\midrule
\multirow{2}{*}{Twitter} 
 & R-1 & \textbf{0.416} & \underline{0.388} & 0.337 & 0.304 & 0.360\\
 & R-L & \underline{0.363} & \textbf{0.369} & 0.315 & 0.269 & 0.319\\
\midrule
\multirow{2}{*}{News} 
 & R-1 & \textbf{0.148} & \underline{0.114} & 0.105 & 0.102 & 0.129\\
 & R-L & \textbf{0.128} & \underline{0.111} & 0.099 & 0.101 & 0.106\\

\midrule
\multirow{2}{*}{Avg.}
 & R-1 & \textbf{0.335} & \underline{0.321} & 0.271 & 0.278 & 0.298\\
 & R-L & \textbf{0.296} & \underline{0.293} & 0.242 & 0.247 & 0.255\\
\bottomrule
\end{tabular}
\caption{Performance comparison of different methods on three LaMP benchmarks using Qwen3-8B as the foundation model. WFT achieves the best average performance across datasets, while remaining competitive with SFT on each individual task. The best-performing results are bolded, and the second-best results are underlined.}
\label{tab:results}

\end{table}

\begin{table}[t]
\centering
\begin{tabular}{llccccc}
\toprule
Dataset & Metric & WFT & SFT & In-Context & Prefix-tuning & CHAMELEON \\
\midrule
\multirow{2}{*}{Paper} 
 & R-1 & \textbf{0.413} & 0.382 & 0.283 & 0.364 & \underline{0.396} \\
 & R-L & \textbf{0.338} & 0.321 & 0.239 & 0.327 & \underline{0.335} \\
\midrule
\multirow{2}{*}{Twitter} 
 & R-1 & \textbf{0.318} & 0.281 & 0.203 & \underline{0.302} & 0.277 \\
 & R-L & \underline{0.260} & 0.242 & 0.167 & \textbf{0.263} & 0.248 \\
\midrule
\multirow{2}{*}{News} 
 & R-1 & \underline{0.131} & 0.115 & 0.109 & \textbf{0.136} & 0.127 \\
 & R-L & \underline{0.113} & 0.095 & 0.090 & 0.111 & \textbf{0.118} \\
\midrule
\multirow{2}{*}{Avg.}
 & R-1 & \textbf{0.287} & 0.259 & 0.198 & \underline{0.267} & \underline{0.267} \\
 & R-L & \textbf{0.237} & 0.219 & 0.165 & \underline{0.234} & \underline{0.234} \\
\bottomrule
\end{tabular}
\caption{Performance comparison of different methods on three LaMP benchmarks using Llama-3.1-8B-Instruct as the foundation model. WFT achieves the best average performance across datasets. The best-performing results are bolded, and the second-best results are underlined.}
\label{tab:llama_results}
\end{table}

\subsection{Results of Budget Control Training}
Figure~\ref{fig:combined} shows the performance of SFT as the budget varies under high and low $N/E$ scenarios, respectively. In both scenarios, as expected, the performance of SFT improves as the budget increases. We do not notice that the performance of this task varies with $N/E$. The performance of WFT is roughly comparable to that of SFT with $B = 1,500-2,000$. As discussed in Section~\ref{sec:comp-analysis}, the effective budget $B$ of WFT is about 100, which is about half of the leftmost point in Figure~\ref{fig:combined}. WFT uses no more than 7\% of the budget required by SFT to achieve comparable performance. Therefore, for personalization tasks, when computational resources are limited such that SFT cannot produce effective results, WFT serves as a strong alternative.
\begin{figure}[htbp]
    \centering
\includegraphics[width=\textwidth]{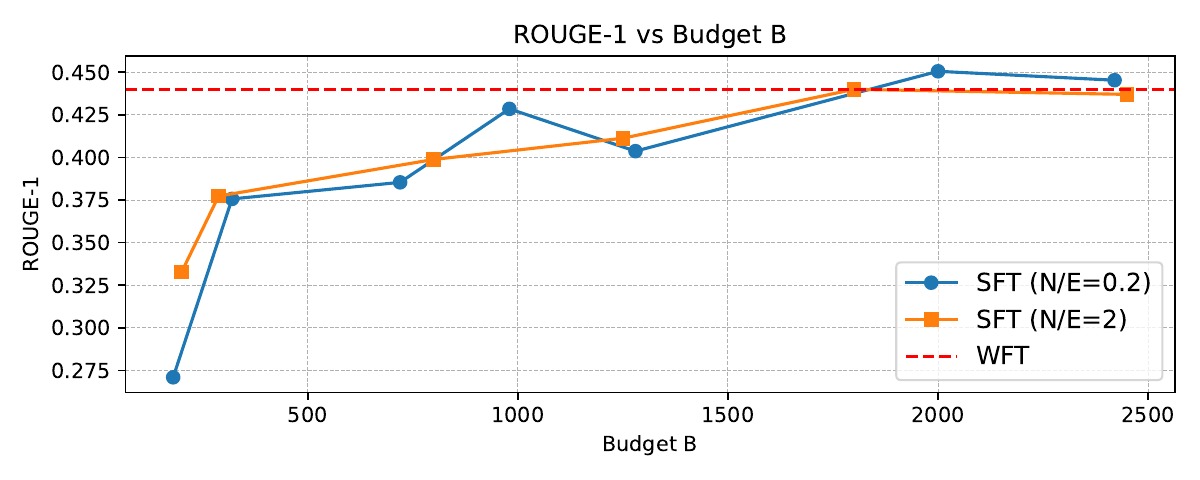}
    \caption{SFT performance as a function of training budget under different training-data-to-epoch ratios ($N/E$). The red dashed line indicates WFT performance.}
    \label{fig:combined}
\end{figure}

\subsection{Ablation Analysis}
We conducted an ablation study by setting M to the identity matrix, reducing the amount of author text used to one-half and one-quarter, and decreasing the number of WFT iteration steps to one-half and one-quarter, in order to examine the effect of each key component on the algorithm.
Table~\ref{tab:ablation} shows that removing the transport operator $M$ causes the largest average drop, confirming that structured transport is central to WFT. Reducing the number of virtual steps also degrades performance, with a larger drop at 1/4 steps than at 1/2 steps, suggesting that multi-step accumulation contributes meaningfully to adaptation quality. Reducing the amount of author history leads to a similar degradation on average, supporting the intuitive idea that using more author data helps personalization tasks. We note that replacing M with the identity matrix does not have a major effect on the Twitter data, which may be because the Twitter training data are similar and related to the task prompts, allowing the residual alone to convey a considerable amount of information.

\begin{table}[t]
\centering
\setlength{\tabcolsep}{5pt}
\begin{tabular}{llcccccc}
\toprule
\multirow{2}{*}{Dataset} & \multirow{2}{*}{Metric} & \multirow{2}{*}{\textbf{Full WFT}} & \multicolumn{1}{c}{Component} & \multicolumn{2}{c}{Steps} & \multicolumn{2}{c}{Data} \\
\cmidrule(lr){4-4} \cmidrule(lr){5-6} \cmidrule(lr){7-8}
 &  &  & w/o $M$ & 1/2 steps & 1/4 steps & 1/2 data & 1/4 data \\
\midrule

\multirow{2}{*}{Paper} 
& R-1 & \textbf{0.440} & 0.350 & 0.406 & 0.372 & 0.388 & 0.348 \\
& R-L & \textbf{0.398} & 0.322 & 0.344 & 0.339 & 0.301 & 0.306 \\

\midrule

\multirow{2}{*}{Twitter} 
& R-1 & \textbf{0.416} & 0.398 & 0.397 & 0.359 & 0.403 & 0.392 \\
& R-L & 0.363 & \textbf{0.368} & 0.363 & 0.296 & 0.354 & 0.333 \\

\midrule

\multirow{2}{*}{News} 
& R-1 & \textbf{0.148} & 0.082 & 0.129 & 0.121 & 0.136 & 0.124 \\
& R-L & \textbf{0.128} & 0.068 & 0.106 & 0.108 & 0.124 & 0.109 \\

\midrule

\multirow{2}{*}{Avg.}
& R-1 & \textbf{0.335} & 0.277 & 0.311 & 0.284 & 0.309 & 0.288 \\
& R-L & \textbf{0.296} & 0.253 & 0.271 & 0.248 & 0.260 & 0.249 \\

\bottomrule
\end{tabular}
\caption{Ablation study of WFT across three LaMP datasets. Removing the transport operator $M$ generally causes the largest degradation, while reducing the number of virtual steps or the amount of author history leads to more gradual performance drops.}
\label{tab:ablation}
\end{table}

\subsection{Human Evaluation}
We conducted a human evaluation to complement the automatic metrics. We randomly sampled 50 examples across the three datasets. For each example, two annotators fluent in English were shown the
target author's historical writing excerpts, the task input, and two anonymized model outputs generated by WFT and SFT. They were asked to judge which output
better matched the author's writing style, considering lexical choice, tone, phrasing, and overall writing pattern. Averaged over the two annotators, WFT was preferred in 31 out of 50 comparisons. This result provides complementary evidence beyond the ROUGE metric that WFT captures author-specific writing style at a level comparable to, and often preferred over, SFT.

\section{Further Analysis and Discussion}
\subsection{Complexity Analysis}\label{sec:comp-analysis}
We theoretically compare the computational cost of WFT and SFT because WFT is designed to approximate SFT in similar usage settings. 

Consider the following variables: let \(E\) denote the number of SFT epochs, \(R\) the number of WFT steps, \(T_f\) the time of a forward pass, and \(T_b\) the time of a backward pass. For WFT, \(K\) denotes the number of forward passes ($K$ different dropout settings). We further use \(S\) to denote the input length, \(V\) the vocabulary size, and \(P\) the total number of model parameters. We assume there is only one training instance; otherwise, it would only introduce an additional constant factor for each author.

For SFT, the computational complexity per author should be: $\mathcal{O}(E(T_f+T_b))$. For WFT, the computational complexity is $\mathcal{O}((KT_f+g))$ , where $g$ is the cost of matrix computations per instance. According to the algorithm detailed in Appendix~\ref{app:acc_update}, we have $g=O(RSK^2V)$, then for WFT the complexity is $\mathcal{O}((KT_f+g)) = \mathcal{O}(K(T_f+RSKV))$.

Suppose \(K\) and \(E\) are of the same order of magnitude (in our experiments, \(K=10\) and \(E=40\), so this assumption is favorable to SFT). Then the key comparison is between \(RSKV\) and \(T_b\). For simplicity, we let \(T_b = cT_f\) and assume \(T_f \propto SP\), where \(S\) is the input length and \(P\) is the number of model parameters. Thus, \(T_b = cSP\), where \(c\) is a constant. For LoRA-SFT, although only a small number of parameters are updated, backward propagation still needs to transmit gradients through the full computation graph. Therefore, its backward cost remains proportional to the model size and input length, and we approximate it using the same form \(T_b = cSP\), potentially with a smaller constant factor than in full-parameter SFT.

With that, we need to compare $cP$ and $RKV$. In experiments, $R=400, V\approx1.5*10^5,K=10$, so $RKV$ is about $6*10^8$ and $cP$ is at least $1.6*10^{10}$ for $c\geq2$~\citep{trainhoff2022,wiedemann2020dithered} if using 8B models. Under these assumptions, WFT has an estimated computation cost 20 times lower than SFT for the considered configuration. This complexity gap is roughly consistent with the difference in training time we observe in practice. The SFT results presented in Table~\ref{tab:results} for the paper task use $B = N*E = 50*40 = 2000$, which means that the effective budget $B$ of WFT is approximately 100.

\begin{figure*}[t]
    \centering
    \resizebox{0.9\textwidth}{!}{
        \begin{minipage}{\textwidth}
            \centering
            \begin{subfigure}[t]{0.55\textwidth}
                \centering
                \includegraphics[width=\textwidth]{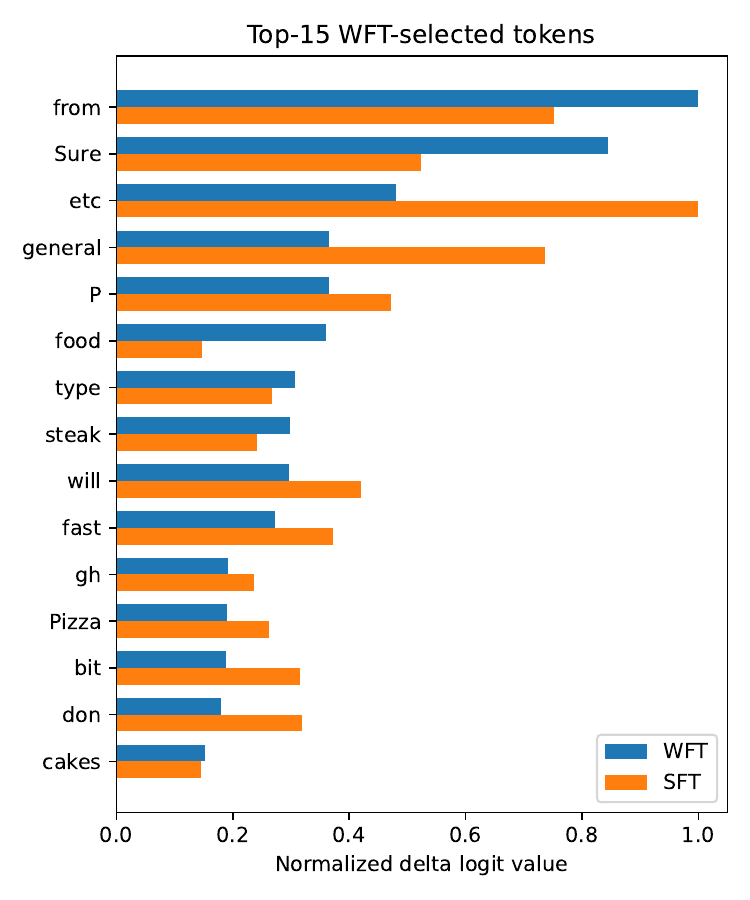}
                \subcaption{Normalized logit shifts of the top-15 tokens ranked by
                \(\Delta z^{\text{WFT}}\), compared between one-step WFT and SFT.}
                \label{fig:topicftsft}
            \end{subfigure}
            \hspace{0.02\textwidth}
            \begin{subfigure}[t]{0.40\textwidth}
                \centering
                \includegraphics[width=0.7\textwidth]{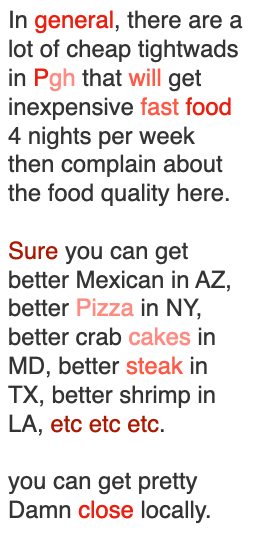}
                \subcaption{Excerpts from the author's historical writings.
                Highlighted tokens are colored according to their rank under
                \(\Delta z^{\text{WFT}}\).}
                \label{fig:qualitative_history}
            \end{subfigure}
        \end{minipage}
    }
    \caption{Qualitative comparison between WFT and SFT on a randomly sampled author.}
    \label{fig:qualitative_case}
\end{figure*}
\subsection{Qualitative Example}\label{sec:qual}
To show how WFT approximates SFT in logit space, we present a case study on a randomly sampled author from the Twitter dataset. Using the same set of the author’s writings, we perform one SFT step and one WFT step, and then apply both updates to the same input prompt. This yields two updated next-token logit vectors for comparison. Let \(\Delta z^{\text{SFT}}\) and \(\Delta z^{\text{WFT}}\) denote the corresponding logit shifts relative to the base model at the same decoding position. 
We compare the two updates on the top-\(k\) tokens ranked by \(\Delta z^{\text{WFT}}\), and compute the cosine similarity between \(\Delta z^{\text{SFT}}\) and \(\Delta z^{\text{WFT}}\) on this subset. The cosine similarity reaches 0.875 on the top-10 tokens, which account for more than 95\% of the next-token probability under both methods. On the top-50 tokens, the similarity still remains 0.601. 

Figure~\ref{fig:topicftsft} further visualizes the normalized logit shifts of the top-15 tokens ranked by \(\Delta z^{\text{WFT}}\) after removing punctuation marks. We observe that the tokens most strongly upweighted by WFT are also assigned positive and large shifts by SFT, indicating a similar relative emphasis on the most affected tokens. Figure~\ref{fig:qualitative_history} further shows that many of these highly upweighted tokens also occur in the sampled author's historical writings. These observations suggest that WFT recovers a logit-update direction similar to that induced by one-step SFT.

\subsection{Other advantages of WFT}
Beyond speed, WFT offers several practical advantages: 1) it does not generate any new weights, so there is no need to store new model weights for each author. 2) the method can easily adapt to new tasks, whereas SFT often requires separate training for different tasks. WFT can also readily incorporate new author data by computing the relevant variables for the additional data, while SFT would require retraining. 3) since WFT does not modify the model parameters, it does not affect the model’s general performance on other tasks. By contrast, tuning on personal data may lead to catastrophic forgetting~\citep{kaushik2021understandingcatastrophicforgettingremembering}. 4) WFT does not require direct access to model weights and operates only through stochastic forward passes and output logits.

\section{Conclusion}
We introduced WFT, a training-free method for personalization that approximates SFT by transporting supervised residuals from training sequences to prompt prefixes in logit space. The cross-prefix transport operator, estimated through dropout-induced forward-pass covariance, supports this transfer without gradient computation or parameter update. Across three tasks, WFT achieves the best average performance across datasets, performs comparably to SFT, and consistently outperforms other lightweight personalization baselines on average. Budget-controlled comparisons further show that WFT can approach the performance of much more expensive SFT while using only a small fraction of the computation. The results suggest that for personalization settings where per-author adaptation must be both effective and lightweight, much of the distributional effect of fine-tuning can be recovered without the fine-tuning itself, making WFT a practical training-free alternative to SFT. WFT operates in logit space and does not reproduce the internal representation
changes induced by parameter fine-tuning. Extending WFT to reasoning-heavy, long-form, and cross-task adaptation is an important direction for future work.

\newpage
\section*{Ethics Statement}
We study personalization in language models and recognize that such settings may involve user-associated text. Our method does not create or store separate user-specific model weights, and instead performs adaptation through inference-time logit corrections, which can reduce the need to maintain personalized parameter copies. In this paper, experiments are conducted on a public benchmark used for research purposes, and we do not introduce new private user data. For real-world deployment, we emphasize that personalization should be applied only with appropriate consent and retention safeguards.

\section*{Acknowledgments}

YW was supported in part by funding from the Office of Naval Research under grant N00014-23-1-2590, the National Science Foundation under grant No. 2310831, No. 2428059, No. 2435696, No. 2440954, a Michigan Institute for Data Science Propelling Original Data Science (PODS) grant, Two Sigma Investments LP, and  LG Management Development Institute AI Research.

\bibliography{colm2026_conference}
\bibliographystyle{colm2026_conference}

\appendix

\section{Formal Justification: Dropout Covariance as an NTK Approximation}
\label{app:ntk-approx}

We now show formally that the dropout-estimated transport operator
$\widehat{M}_{t \leftarrow s}$ is a ridge-regularized approximation to the
cross-prefix neural tangent kernel (NTK) action.  This bridges the gap
between the logit-space derivation in Sections~3.1--3.2 and standard
NTK theory, making precise the sense in which WFT approximates one-step
SFT.

\begin{proposition}[Dropout Covariance Approximates the Cross-Prefix NTK Action]
\label{prop:ntk}
Let $z(c;\theta)\in\mathbb{R}^V$ denote the logits of a language model with
parameters $\theta\in\mathbb{R}^P$ for prefix $c$, and let
$J_c = \frac{\partial z(c;\theta)}{\partial \theta}\Big|_{\theta_0}
       \in \mathbb{R}^{V \times P}$
be the Jacobian at the frozen parameters $\theta_0$.
Define the cross-prefix and self-prefix NTK Gram matrices:
\[
  \Theta_{t,s} \;=\; J_t\, J_s^\top \;\in\;\mathbb{R}^{V\times V},
  \qquad
  \Theta_{s,s} \;=\; J_s\, J_s^\top \;\in\;\mathbb{R}^{V\times V}.
\]
When \(\Theta_{s,s}\) is singular, its inverse below is understood as the
Moore--Penrose pseudoinverse.

Suppose the following conditions hold:
\begin{enumerate}
  \item \textbf{(Local linearity.)} 
    The logit function is well-approximated to first order around 
    $\theta_0$: for every dropout-induced effective parameter 
    $\theta_0 + \delta\theta_k$,
    \[
      z^{[k]}(c)
      \;=\; z(c;\,\theta_0 + \delta\theta_k)
      \;\approx\; z(c;\theta_0) + J_c\,\delta\theta_k.
    \]

  \item \textbf{(Isotropic dropout perturbations.)} 
    The population covariance of the dropout-induced parameter 
    perturbations satisfies
    $\Sigma_\theta \;=\; \mathrm{Cov}(\delta\theta) \;=\; \sigma^2 I_P$
    for some $\sigma^2 > 0$.
\end{enumerate}
Then:

\medskip\noindent
\textbf{(a)} \emph{(NTK transport.)}
For $r_s$ in the column span of $\Theta_{s,s}$, the minimum-norm
parameter update that realizes the source-side logit update $\eta r_s$
induces the following logit shift at prompt prefix $x_{1:t}$:
\[
  \Delta z_t^{\mathrm{SFT}}
  \;=\; \eta\, \Theta_{t,s}\,\Theta_{s,s}^{-1}\, r_s.
\]
The corresponding cross-prefix transport operator is
$M_{t\leftarrow s} = \Theta_{t,s}\,\Theta_{s,s}^{-1}$.

\medskip\noindent
\textbf{(b)} \emph{(Dropout estimation.)}
The dropout-estimated operator satisfies, as $K\to\infty$,
\[
  \widehat{M}_{t\leftarrow s}
  \;\xrightarrow{\;p\;}
  \Theta_{t,s}\bigl(\Theta_{s,s} + \lambda' I\bigr)^{-1},
  \qquad
  \lambda' = \frac{\lambda}{\sigma^2},
\]
which is a ridge-regularized approximation to $M_{t\leftarrow s}$.
In particular, for any residual $r_s$ in the column span of $\Theta_{s,s}$,
\[
  \widehat{M}_{t\leftarrow s}\, r_s
  \;\longrightarrow\;
  M_{t\leftarrow s}\, r_s
  \qquad \text{as } \lambda' \to 0.
\]
\end{proposition}

\begin{proof}
\textbf{Part (a).}
By the KL-prox NGD formulation (Section~3), a natural-gradient step
at training position $s$ with step size $\eta$ produces the logit-space
update $\eta\, r_s$ at position $s$, i.e.,
$z_s^{\mathrm{new}} = z_s + \eta\, r_s$.

In the NTK (lazy-training) regime, the logit function is locally linear in
$\theta$, so a parameter displacement $\Delta\theta$ induces
$\Delta z(c) = J_c\,\Delta\theta$ at any prefix $c$.
The parameter update $\Delta\theta$ that realises the prescribed logit shift
$\eta\, r_s$ at position $s$ while minimising $\|\Delta\theta\|^2$ is:
\[
  \Delta\theta
  \;=\; \eta\, J_s^\top\bigl(J_s\, J_s^\top\bigr)^{-1} r_s
  \;=\; \eta\, J_s^\top\, \Theta_{s,s}^{-1}\, r_s,
\]

which is the Moore--Penrose pseudoinverse solution to
$J_s\,\Delta\theta = \eta\, r_s$.

The induced logit shift at the prompt prefix $x_{1:t}$ is then:
\[
  \Delta z_t^{\mathrm{SFT}}
  \;=\; J_t\,\Delta\theta
  \;=\; \eta\, J_t\, J_s^\top\, \Theta_{s,s}^{-1}\, r_s
  \;=\; \eta\, \Theta_{t,s}\,\Theta_{s,s}^{-1}\, r_s.
\]
This identifies the corresponding minimum-norm transport operator as
$M_{t\leftarrow s} = \Theta_{t,s}\,\Theta_{s,s}^{-1}$.

\medskip
\textbf{Part (b).}
Under the local linearity assumption, the centred dropout deviations are:
\[
  t_k = z^{[k]}(x_{1:t}) - \bar{z}(x_{1:t})
  \;\approx\; J_t\,\delta\theta_k,
  \qquad
  s_k = z^{[k]}(\tilde{x}_{1:s}) - \bar{z}(\tilde{x}_{1:s})
  \;\approx\; J_s\,\delta\theta_k.
\]
The empirical cross-covariance and self-covariance concentrate around
their population counterparts as $K\to\infty$:
\begin{align*}
  \Sigma_{t,s}
  &= \frac{1}{K-1}\sum_{k=1}^K t_k\, s_k^\top
  \;\xrightarrow{\;p\;}\;
  J_t\,\Sigma_\theta\, J_s^\top
  \;=\; \sigma^2\, \Theta_{t,s}, \\[4pt]
  \Sigma_{s,s}
  &= \frac{1}{K-1}\sum_{k=1}^K s_k\, s_k^\top
  \;\xrightarrow{\;p\;}\;
  J_s\,\Sigma_\theta\, J_s^\top
  \;=\; \sigma^2\, \Theta_{s,s}.
\end{align*}
Substituting into the ridge-regularized estimator:
\begin{align*}
  \widehat{M}_{t\leftarrow s}
  &= \Sigma_{t,s}\bigl(\Sigma_{s,s} + \lambda I\bigr)^{-1} \\
  &\xrightarrow{\;p\;}\;
    \sigma^2\,\Theta_{t,s}
    \bigl(\sigma^2\,\Theta_{s,s} + \lambda I\bigr)^{-1} \\
  &= \Theta_{t,s}
    \bigl(\Theta_{s,s} + \tfrac{\lambda}{\sigma^2}\, I\bigr)^{-1} \\
  &= \Theta_{t,s}
    \bigl(\Theta_{s,s} + \lambda' I\bigr)^{-1}.
\end{align*}
For any $r_s$ in the column span of $\Theta_{s,s}$, taking $\lambda'\to 0$
recovers $\Theta_{t,s}\,\Theta_{s,s}^{-1} r_s = M_{t\leftarrow s}\, r_s$.
\end{proof}

\begin{remark}[Role of ridge regularization]
In practice $K \ll V$, so $\Sigma_{s,s}$ has rank at most $K-1$ and the
unregularized inverse does not exist in $\mathbb{R}^{V\times V}$.
The ridge term $\lambda I$ simultaneously (i)~ensures numerical stability,
(ii)~acts as a spectral filter that suppresses directions poorly estimated
by the finite dropout sample, and (iii)~provides implicit regularization
analogous to early stopping in gradient descent.
\end{remark}

\begin{remark}[Connection to the qualitative analysis]
Proposition~\ref{prop:ntk} predicts that
$\widehat{M}_{t\leftarrow s}\, r_s$ should align with the logit shift
produced by one SFT step. The qualitative example shown in~\ref{sec:qual} provides empirical support for this prediction: the practical dropout-based approximation captures logit-update directions that align with those induced by one-step SFT over the subspace containing more than 95\% of the next-token probability mass.
\end{remark}

\section{Logits Update Implementation and Complexity Analysis}
\label{app:acc_update}

For a training sequence of length $S$, we run $K$ stochastic forward passes
(with different dropout masks) through the frozen model, obtaining logit
vectors at each source position $s$. Following the notation in
Sec.~\ref{sec:M}, we form the centered deviations
\[
    u_s^{(k)}
    =
    z^{[k]}(\tilde{x}_{1:s})
    -
    \overline{z}(\tilde{x}_{1:s}),
\]
and collect them into
$U_s = [u_s^{(1)};\dots;u_s^{(K)}] \in \mathbb{R}^{K \times V}$.
The mean logit $\bar{z}_s$ is used to initialize
$z_s^{(0)} = \bar{z}_s$, with initial residual
\[
    r_s^{(0)}
    =
    y_s - \mathrm{softmax}(z_s^{(0)}).
\]

Then at each step $j = 0, 1, \dots, R-1$:
\begin{enumerate}

\item \textbf{Compute the transport-weighted update direction.}
For each source position $s$, define the empirical self-covariance
\[
    \Sigma_{s,s}
    =
    \frac{1}{K-1}U_s^\top U_s
    \in \mathbb{R}^{V\times V},
\]
and the ridge-regularized weight vector
\[
    w_s^{(j)}
    =
    \left(\Sigma_{s,s}+\lambda I_V\right)^{-1}
    r_s^{(j)}
    \in \mathbb{R}^{V}.
\]
The matrix $\Sigma_{s,s}$ is not materialized in the implementation.
Instead, $w_s^{(j)}$ is computed through the dual Gram matrix
\[
    G_s
    =
    \frac{1}{K-1}U_sU_s^\top
    \in \mathbb{R}^{K\times K}
\]
using the Woodbury identity, as detailed below.

For the source-to-target interaction, the implementation uses a
vocabulary-coordinate-wise approximation. We aggregate the source-side
quantities into a shared coefficient matrix
$B^{(j)} \in \mathbb{R}^{K\times V}$:
\[
    B_{k,v}^{(j)}
    =
    \frac{1}{S}
    \sum_{s=1}^{S}
    U_s(k,v)\,w_s^{(j)}(v).
\]
This matrix summarizes the source-side update associated with each
dropout direction and vocabulary coordinate.

\item \textbf{Update source logits and refresh residuals.}
The source logits are updated directly by the residual:
\[
    z_s^{(j+1)}
    =
    z_s^{(j)} + \eta\,r_s^{(j)},
\]
and the residual is recomputed exactly from the updated logits:
\[
    r_s^{(j+1)}
    =
    y_s-\mathrm{softmax}(z_s^{(j+1)}),
\]
as mentioned in Sec.~\ref{sec:M}.

\end{enumerate}

During the $R$-step trajectory, the coefficient matrices $B^{(j)}$ are
computed. In practice, we retain $B^{(j)}$ every 20 steps to reduce memory
and computation.

\paragraph{Inference.}
At inference, for each generated token at position $t$, we run $K$ forward
passes on the current prefix to obtain target tangents
$T_0=[t_1;\dots;t_K]\in\mathbb{R}^{K\times V}$.
Using the stored coefficient matrices, the accumulated logit shift from
Sec.~\ref{sec:M} is evaluated as
\[
    \Delta z_t^{(R)}
    =
    \eta
    \sum_{j=0}^{R-1}
    \frac{1}{K-1}
    \sum_{k=1}^{K}
    t_k\odot B_k^{(j)}.
\]

\paragraph{Efficient computation and complexity.}
The full estimated operator motivating the implementation is
\[
    \widehat{M}_{t\leftarrow s}r_s
    =
    \Sigma_{t,s}
    \left(\Sigma_{s,s}+\lambda I_V\right)^{-1}
    r_s.
\]
Denote the ridge-regularized source-side vector by
\[
    w_s
    =
    \left(\Sigma_{s,s}+\lambda I_V\right)^{-1}
    r_s
    \in\mathbb{R}^{V}.
\]
The implementation retains the full ridge-regularized self-covariance
action through $w_s$, while applying the cross-prefix interaction
coordinate-wise as described above. Neither $\Sigma_{s,s}$ nor
$\Sigma_{t,s}$ is explicitly materialized.

The bottleneck is computing $w_s$: naively inverting
$\Sigma_{s,s}\in\mathbb{R}^{V\times V}$ costs $O(V^3)$, which is
infeasible for large vocabularies. Since
\[
    \Sigma_{s,s}
    =
    \frac{1}{K-1}U_s^\top U_s,
\]
the Woodbury matrix identity~\citep{woodbury1950inverting} gives
\[
    w_s
    =
    \frac{1}{\lambda}r_s
    -
    \frac{1}{\lambda^2(K-1)}
    U_s^\top
    \left(
        I_K
        +
        \frac{1}{\lambda(K-1)}
        U_sU_s^\top
    \right)^{-1}
    U_s r_s.
\]
This requires only solving a system involving the $K\times K$ matrix
\[
    I_K
    +
    \frac{1}{\lambda(K-1)}
    U_sU_s^\top
\]
at cost $O(K^3)$, together with matrix-vector products costing
$O(K^2V)$ per source position. Since $K<S$, summing over all $S$
positions gives $O(SK^2V)$ per step, which dominates the other operations
(aggregating $B^{(j)}$: $O(SKV)$; computing residuals: $O(SV)$).
Therefore, the total offline cost over $R$ steps is
\[
    g
    =
    O\!\left(
        R\cdot S\cdot K^2\cdot V
    \right).
\]
This cost involves no gradient computation; the model parameters remain
frozen throughout.

\section{Experiment Details}
\subsection{Dataset}\label{app:data}
Because the training cost of SFT is high, we did not perform SFT on all 1K+ authors in each LaMP dataset. Instead, for each dataset, we randomly selected 50 authors from the original test authors as the new test authors. In addition, WFT does not use the original LaMP training data because it is training-free. We use 10 authors from the dev set to select hyperparameters. SFT also does not use the original training data because our SFT is performed per author. We use each test author’s historical data as the SFT training data, and do not require data from other authors in the original training set. Prefix tuning, however, uses the original training data in LaMP to train the prefix network, since we hope that the patterns learned from the training authors can generalize to new authors and help achieve personalization.

Here are training texts for different datasets:

Paper Title: "Generate a title for the following abstract: [ABSTRACT].
Title: [TITLE]"

News Headline: "Generate a headline for the following article: [ARTICLE].
Headline: [HEADLINE]"

Since the Twitter data only provides authors’ original tweets and does not include paraphrase ground truth, the training input consists of batches of writing excerpts only for each author. The content in brackets above will be replaced with the actual author data.
\subsection{Methods}\label{app:methods}
\subsubsection{WFT}

We run $K=10$ dropout-perturbed forward passes per sequence. The offline trajectory is computed for $R=400$ steps with $\eta = 5\times 10^{-3}$ and ridge regularization
$\lambda = 10^{-4}$. To reduce memory, the coefficient matrices $B^{(j)}$
are stored every 20 steps rather than at every step. $\eta$ and $\lambda$ were selected by grid search on a held-out set of development authors. Specifically, we
searched $\eta \in \{10^{-4},5\times10^{-4},10^{-3}, 5\times10^{-3}, 10^{-2}\}$,
$\lambda \in \{10^{-5}, 10^{-4}, 10^{-3}\}$. $K$ and $R$ were
fixed based on computational budget.

\subsubsection{SFT}\label{app:sft}
All SFT experiments fine-tune the foundation models using LoRA~\citep{hu2022lora}
with rank $r=8$, scaling factor $\alpha=32$. For the full experiments, the model is trained for 40 epochs on up to 50 author-specific training
sequences, with batch size 1 and gradient accumulation over 32 steps. We use the AdamW optimizer with a learning rate of $10^{-3}$. The model is trained in bfloat16 precision. The learning rate was selected by grid search on a held-out set of development authors over $\{10^{-5}, 10^{-4},10^{-3}, 10^{-2}\}$.
\subsubsection{Prefix-tuning}\label{app:prefix-tuning}

Following previous work~\citep{liu-etal-2023-recap,huber2025embeddingtoprefixparameterefficientpersonalizationpretrained}, our prefix-tuning method consists of two stages: a style encoder trained
offline across a large volume of authors, and a prefix decoder that maps the style embedding
to per-layer key-value prefix tensors injected at inference time.

We train a style encoder on top of a frozen RoBERTa-base~\citep{liu2019roberta} model.
The encoder takes an author's text segments as input, applies mean pooling
over the last hidden states, and passes the result
through a two-layer projection head (hidden dimension 512, output dimension
256, GELU activation~\citep{hendrycks2016gaussian}) with $\ell_2$ normalization.
The encoder is trained with supervised contrastive
loss~\citep{khosla2020supervised} at temperature $\tau=0.07$, using batches of $B=32$ authors with $K=4$ segments each.
We use AdamW with learning rate $10^{-4}$ and weight decay $0.01$ for
100 epochs. The checkpoint with the best author retrieval recall@1 on a
held-out author set is selected.

Then, a three-layer MLP with Tanh activation maps the 256-dimensional
style embedding to key-value prefix tensors of length $L_p$ for all attention layers of the frozen foundation models, shared across layers.
The prefix decoder is trained end-to-end on the generation objective
with learning rate $2\times 10^{-4}$, using prefix lengths
$L_p=32$.

\subsubsection{In-Context Prompting}\label{app:icp}
We experimented with varying the number of excerpts (up to $k=20$ per author) and found that performance peaked with five excerpts. We speculate that as the number of excerpts increases, they may not provide additional useful signals about the author’s style and may also exceed the input length that the model can effectively handle. Following~\citet{chakrabarty2025readerspreferoutputsai}, these excerpts are concatenated into a prompt together with the task query and a system message instructing the model to imitate the author's style. Here is an example of an input prompt for the Twitter task:
\begin{quote}
\texttt{You are a careful literary style imitator. Learn stylistic signals from the provided Author Excerpts (lexicon, syntax, cadence, rhetoric, tone, macro-structure) and paraphrase the given tweet in that author's voice. Output the piece only.\\
Rules:\\
- Produce a single coherent piece matching the style suggested by the excerpts.\\
- Do not quote any excerpt verbatim.\\
- No prefaces, no analysis, no headings---just the final prose.}

\medskip
\texttt{Author Excerpts:\\
H1: [TEXT1]\\
H2: [TEXT2]\\
$\cdots$\\
H5: [TEXT5]\\}

\medskip
\texttt{Paraphrase the following tweet: i deleted the app again. we'll see how long this lasts.\\
Output:}
\end{quote}

\subsubsection{CHAMELEON}

CHAMELEON introduces two different approaches for latent space editing, consisting of individual and group alignment method. In Table \ref{tab:results}, we compared CHAMELEON's individual alignment method with WFT, since its group alignment method only introduces marginal performance lift on the validation and test datasets in LaMP. In our experiment, we randomly sampled top-20 history items for each user, built 8 synthetic responses for each condition: personalized prompts and non-personalized prompts, and conducted latent space editing for all MLP layers within the base model. Since the original CHAMELEON only provided the prompts for Twitter dataset which contains dataset-specific key words such as "the tweet", we replaced the key words to support personalization for Paper and News dataset, preventing confusing the LLMs. We evaluated the various configurations to derive the subspace directions for personalized and neutral responses, and recognized that the proposed configuration that employs SVD for personalized direction and CSS identification for neutral direction achieves the optimal performance.

Generation of all methods mentioned above uses top-$p$ sampling with $p=0.8$, temperature
$0.7$, and top-$k=20$.
\end{document}